\documentclass{article}

\PassOptionsToPackage{numbers, compress}{natbib}

\usepackage[preprint]{neurips_2026}
\usepackage{multirow}
\usepackage[utf8]{inputenc} 
\usepackage[T1]{fontenc}    
\usepackage{hyperref}       
\usepackage{url}            
\usepackage{booktabs}       
\usepackage{amsfonts}       
\usepackage{nicefrac}       
\usepackage{microtype}      
\usepackage{xcolor}         
\usepackage{tikz}
\usetikzlibrary{shapes.geometric, arrows.meta, positioning, fit, backgrounds, calc}
\usepackage{enumitem}       

\usepackage{graphicx}      
\usepackage{amsmath, amssymb, amsthm} 
\usepackage{mathtools}      
\usepackage{algorithm}
\usepackage{algorithmic}
\usepackage{wrapfig}
\usepackage{caption}
\usepackage{subcaption}

\theoremstyle{plain}
\newtheorem{theorem}{Theorem}[section]
\newtheorem{proposition}[theorem]{Proposition}
\newtheorem{lemma}[theorem]{Lemma}

\theoremstyle{definition}
\newtheorem{definition}{Definition}[section]
\newtheorem{assumption}{Assumption}[section]
\theoremstyle{remark}
\newtheorem{remark}{Remark}[section]

\newcommand{\R}{\mathbb{R}}
\newcommand{\E}{\mathbb{E}}
\newcommand{\I}{\mathcal{I}}
\newcommand{\bx}{\mathbf{x}}
\newcommand{\pa}{\mathrm{pa}}
\newcommand{\ch}{\mathrm{ch}}
\newcommand{\Cov}{\mathrm{Cov}}
\newcommand{\Var}{\mathrm{Var}}

\title{Optimization-Free Topological Sort for Causal Discovery via the Schur Complement of Score Jacobians}

\author{%
  Rui Wu \\
  School of Computer Science and Engineering\\
  University of Science and Technology of China\\
  \texttt{wurui22@mail.ustc.edu.cn} \\
  \And
  Hong Xie\thanks{Corresponding author.} \\
  School of Computer Science and Engineering\\
  University of Science and Technology of China\\
  \texttt{hongx87@ustc.edu.cn}
}

\begin{document}

\maketitle

\begin{abstract}
Continuous causal discovery typically couples representation learning with structural optimization via non-convex acyclicity penalties, which subjects solvers to local optima and restricts scalability in high-dimensional regimes. We propose a decoupled paradigm that shifts the causal discovery bottleneck from non-convex optimization to statistical score estimation. We introduce the Score-Schur Topological Sort (SSTS), an algorithm that extracts topological order directly from unconstrained generative models, bypassing constrained structure optimization. We establish that the causal hierarchy leaves a geometric signature within the score function: iterative graph marginalization is mathematically equivalent to computing the Schur complement of the Score-Jacobian Information Matrix (SJIM) under linear conditions. This translates the acyclicity constraint into an algebraic procedure with a dominant cost of $\mathcal{O}(d^3)$ operations. For non-linear systems, we formulate the expectation gap of Schur marginalization and introduce Block-SSTS to compress extraction depth, bounding structural error. Empirically, SSTS allows causal structural analysis on non-linear graphs up to $d=1000$. At this scale, our framework indicates that once the non-convex optimization bottleneck is mathematically bypassed, the structural fidelity of continuous causal discovery is bounded by the finite-sample estimation variance of the global score geometry. By reducing graph extraction to matrix operations, this work reframes scalable causal discovery from a constrained optimization problem to a statistical estimation challenge.
\end{abstract}

\section{Introduction}
\label{sec:intro}

Discovering the underlying directed acyclic graph (DAG) from observational data is a core problem in statistical machine learning \citep{pearl2009causality, spirtes2000causation}. Combinatorial search methods are constrained by exponential scaling. \citet{Zheng2018NOTEARS} introduced NOTEARS, which established a continuous algebraic characterization of acyclicity ($\mathrm{tr}(e^{W \circ W}) = d$), enabling gradient-based optimization algorithms \citep{Bello2022DAGMA}.

While offering computational advantages over search algorithms, continuous constrained optimization faces a persistent bottleneck: acyclicity penalties are non-convex. When applied to high-dimensional or non-linear data, optimizers converge to local minima, producing structural false discoveries. Mitigating this necessitates hyperparameter tuning alongside heuristic post-hoc thresholding.

Score-based generative models \citep{hyvarinen2005estimation, song2019generative, ho2020denoising} have been investigated for causal discovery \citep{Rolland2022}. The Jacobian of the score function (i.e., the Hessian of the log-density) contains sufficient information to identify topological leaf nodes. Identifying subsequent nodes requires retraining the score model on marginalized subsets or designing masked architectures, restricting scalability.

We present the \textbf{Score-Schur Topological Sort (SSTS)}, which extracts causal topological order directly from the score function of an unconstrained generative model. By decoupling representation learning from topological extraction, we demonstrate that causal hierarchy leaves an extractable algebraic signature in the geometry of the score function $\nabla_{\bx} \log p(\bx)$. Our contributions are:
\begin{enumerate}[leftmargin=*, noitemsep, topsep=0pt]
    \item \textbf{Algebraic Mapping of Acyclicity (Linear Exactness):}
    In linear Gaussian ANMs, we establish an exact equivalence between iterative leaf-node marginalization and the Schur complement of the Score-Jacobian Information Matrix (SJIM). 
    For non-linear ANMs, we characterize the non-commutativity between expectation and Schur marginalization via an explicit expectation gap.
    \item \textbf{Decoupled Extraction:} By bypassing the non-convex optimization bottleneck, SSTS reframes high-dimensional causal discovery from combinatorial search to pure statistical estimation.
    \item \textbf{Characterization of Non-linear Boundaries:} We formulate the algebraic expectation gap of Schur marginalization in non-linear Additive Noise Models (ANMs). To bound this structural error accumulation, we introduce Block-SSTS, which compresses the extraction depth and scales causal recovery up to $d=1000$ variables in experiments.
\end{enumerate}

\begin{figure*}[t]
\centering
\begin{tikzpicture}[
    >=stealth,
    box/.style={rectangle, rounded corners=4pt, draw=darkgray!60, thick, align=center, text width=3.4cm, minimum height=1.6cm, font=\small},
    arrow/.style={->, thick, color=darkgray, >=Triangle},
    groupbox1/.style={rectangle, rounded corners=6pt, draw=gray!50, dashed, thick, inner sep=12pt},
    groupbox2/.style={rectangle, rounded corners=6pt, draw=gray!50, dashed, thick, inner sep=12pt, minimum width=\textwidth-\pgflinewidth}
]

\node[box, fill=cyan!5] at (2.6, 0) (data) {
    \textbf{1. Input Data}\\[1.5ex] 
    $\mathbf{X} \in \mathbb{R}^{N \times d}$ \\ 
    \scriptsize ANM Manifold
};

\node[box, fill=orange!5] at (8.2, 0) (score) {
    \textbf{2. Pre-trained Score}\\[1.5ex] 
    $\min_\theta \mathbb{E} \| s_\theta - \nabla \log p \|^2_2$ \\ 
    \scriptsize $\nabla_{\mathbf{x}} s_\theta$ Extraction
};

\node[box, fill=green!5] at (0, -4) (SJIM) {
    \textbf{3. Sparse Hessian}\\[1.5ex] 
    $\hat{\mathcal{I}} \approx \frac{1}{N} \sum -\nabla_{\mathbf{x}} s_\theta$ \\ 
    \scriptsize w/ Group Lasso $\ell_{1,2}$
};

\node[box, fill=green!5] at (5.4, -4) (ssts) {
    \textbf{4. Score-Schur Sort}\\[1.5ex] 
    Block Marginalization \\ 
    \scriptsize $\arg\min (\hat{\mathcal{I}}_{ii}) \Rightarrow \mathrm{Leaf}$
};

\node[circle, draw=darkgray!60, thick, fill=white, align=center, minimum size=1.5cm, font=\small] at (10.8, -4) (dag) {
    \textbf{DAG} $\mathcal{G}$
};

\draw[arrow] (data) -- node[above, align=center, font=\scriptsize, color=darkgray] {Representation\\Learning} (score);
\draw[arrow] (SJIM) -- node[above, align=center, font=\scriptsize, color=darkgray] {Algebraic\\Extraction} (ssts);
\draw[arrow] (ssts) -- node[above, align=center, font=\scriptsize, color=darkgray] {Lasso\\Pruning} (dag);
\draw[arrow, rounded corners=6pt] (score.south) -- (8.2, -2.0) -- node[above, align=center, font=\scriptsize, color=darkgray] {Geometric Signature} (0, -2.0) -- (SJIM.north);
\draw[arrow, dashed] (ssts.south) ++(-0.3,0) arc (180:360:0.3 and 0.4) node[below=0.1cm, midway, font=\scriptsize, color=darkgray] {$\mathcal{O}(d)$ Iters};

\begin{scope}[on background layer]
    \node[groupbox1, fill=gray!5, fit=(data)(score)] (stage1) {};
    \node[above=1ex of stage1, font=\bfseries\small, color=darkgray] {Stage 1: Generative Modeling};
    \node[groupbox2, fill=gray!5] at (5.4, -4) (stage2) {};
    \node[above=0.9cm of stage2, font=\bfseries\small, color=darkgray] {Stage 2: Algebraic Structure Discovery};
\end{scope}
\end{tikzpicture}
\caption{\textbf{Decoupled Architecture of SSTS.} Stage 1 optimizes an unconstrained density estimator to capture the data manifold. Stage 2 executes a structure-optimization-free algebraic extraction. Mapping iterative leaf-node marginalization as the Schur complement of the estimated Score-Jacobian Information Matrix replaces constrained acyclicity optimization with deterministic algebraic elimination (exact under linear Gaussian conditions, approximate otherwise).}
\label{fig:pipeline}
\end{figure*}
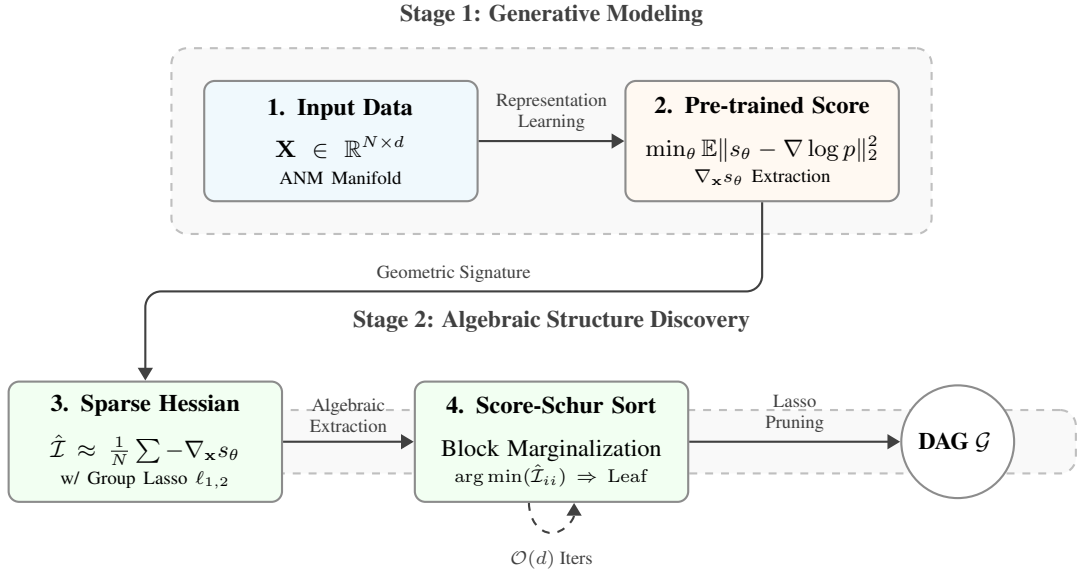

\section{Related Work}
\label{sec:related_work}

\textbf{Continuous Causal Discovery.} 
The transition from combinatorial search to continuous optimization advanced causal discovery. Following NOTEARS \citep{Zheng2018NOTEARS}, efforts expanded continuous acyclicity constraints to non-linear neural networks \citep{lachapelle2020gradient}, graph neural networks \citep{yu2019dag}, and exact likelihood frameworks \citep{ng2020role}. Recent advancements integrated differentiable penalties into Bayesian structure learning \citep{lorch2021differentiable} and log-determinant characterizations \citep{Bello2022DAGMA}. Because these paradigms inherently couple graph discovery with parameter estimation, they suffer from non-convex penalty landscapes. Gradient-based optimizers are thus trapped in local optima when scaling to high-dimensional datasets.

\textbf{Order-Based Search and Generative Models.} 
Topological order search reduces the DAG identification space to permutations, enabling algorithms to evaluate edges systematically under continuous Additive Noise Models (ANMs) \citep{hoyer2008nonlinear, peters2014causal}. Generative models, including normalizing flows \citep{khemakhem2021causal} and diffusion models \citep{ho2020denoising, song2019generative}, have been applied to structure learning. Specifically, score matching techniques isolate leaf nodes from the data log-density \citep{Rolland2022}. Extensions of this framework \citep{montagna2023scalable, sanchez2023diffusion}, nevertheless, continue to rely on recursive network retraining or diffusion-guided structural search to sequence the remaining nodes, creating severe computational bottlenecks.

\textbf{Our Distinction.} 
Existing score-based paradigms rely on iterative model retraining or heuristic search to sequentially isolate nodes. In contrast, SSTS offers a deterministic algebraic extraction procedure. By mapping continuous graph marginalization to the Schur complement of a single pre-trained Score-Jacobian Information Matrix\citep{lauritzen1996graphical, friedman2008sparse}, SSTS decouples generative representation learning from the topological sort. Non-convex acyclicity optimization is bypassed, and the hierarchy is obtained via deterministic matrix operations (exact in linear Gaussian regimes and approximate otherwise).

\section{Theoretical Foundation: Score-Jacobian Geometry and Topology}
\label{sec:theory}
In linear Gaussian ANMs, the topological ordering is recoverable from the SJIM via Schur-complement-based elimination. 
For non-linear ANMs, the same elimination on the expected SJIM introduces an expectation gap, which we characterize below.

We consider observational data generated by a continuous Additive Noise Model (ANM) \citep{shimizu2006linear, hoyer2008nonlinear}. Let $\mathcal{G} = (\mathcal{V}, \mathcal{E})$ be a DAG over $d$ variables. The structural equations are:
\begin{equation} \label{eq:anm}
    x_i = f_i(\bx_{\pa(i)}) + \epsilon_i, \quad i = 1, \dots, d.
\end{equation}
where $\pa(i)$ denotes the parents of node $i$, $f_i$ are smooth, twice-differentiable functions, and $\epsilon_i$ are mutually independent noise variables. 

\begin{assumption}[Homoscedastic Gaussian Noise]
\label{assum:homo_noise}
Noise components $\epsilon_i \sim \mathcal{N}(0, \sigma^2)$ are Gaussian with variance $\sigma^2 > 0$. The conditional log-likelihood is $\log p_{\epsilon_i}(\epsilon_i) = -\frac{\epsilon_i^2}{2\sigma^2} - \frac{1}{2}\log(2\pi\sigma^2)$.
\end{assumption}

\begin{assumption}[Causal Faithfulness and Non-degeneracy]
\label{assum:faithfulness}
For any directed edge $j \to i \in \mathcal{E}$, the structural equation $f_i$ is non-degenerate with respect to $x_j$, satisfying: $\E_{p(\bx)} \left[ \left( \frac{\partial f_i}{\partial x_j} \right)^2 \right] > 0$.
\end{assumption}

Let $s(\bx) = \nabla_{\bx} \log p(\bx)$ be the true score function. We define $\I \in \R^{d \times d}$ as the Expected Score-Jacobian Information Matrix (SJIM), formulated as the expected negative Hessian of the log-density:
\begin{equation}
    \I = \E_{p(\bx)} \left[ - \nabla_{\bx}^2 \log p(\bx) \right].
\end{equation}
Under the strictly additive noise formulation (Assumption \ref{assum:homo_noise}), each conditional distribution $p(x_i \mid \bx_{\pa(i)})$ forms a location family parameterized by $\bx_{\pa(i)}$. In this regime, the SJIM corresponds to the Fisher Information Matrix with respect to these location parameters. Hereafter, we refer to $\I$ as the structural information matrix.

\subsection{Exact Leaf Identifiability via Diagonal Energy}
Building upon the leaf identifiability principles of score-based causal discovery \citep{Rolland2022}, we show that the diagonal elements of $\I$ explicitly encode the DAG's topological hierarchy.

\begin{theorem}[Leaf Node Identifiability]
\label{thm:leaf_discovery}
Under the ANM defined in Eq. \ref{eq:anm} satisfying Assumption \ref{assum:homo_noise}, the diagonal entries of the Score-Jacobian Information Matrix are given by:
\begin{equation}
    \I_{ii} = \frac{1}{\sigma^2} + \frac{1}{\sigma^2} \sum_{j \in \ch(i)} \E_{p(\bx)} \left[ \left( \frac{\partial f_j}{\partial x_i} \right)^2 \right].
\end{equation}
Consequently, if and only if node $l$ is a topological sink (leaf node), $\ch(l) = \emptyset$, and it attains the minimum diagonal value: $\I_{ll} = \min_{i} \I_{ii} = 1/\sigma^2$. (Formal proof is deferred to Appendix~\ref{app:proofs_leaf}).
\end{theorem}
\begin{remark}[Robustness to Heteroscedasticity]
\label{rmk:heteroscedasticity}
While Theorem \ref{thm:leaf_discovery} assumes homoscedasticity, absolute diagonal energy $\I_{ii}$ exhibits scale sensitivity under heteroscedastic noise. In Appendix \ref{app:heteroscedasticity}, we extend this formulation by deriving a scale-invariant Relative Diagonal Variance ($CV^2$) metric, which maintains exact algebraic identifiability across varying noise regimes.
\end{remark}

\subsection{Exact Marginalization via Schur Complement in Linear ANMs}
Once a leaf node $l$ is identified, sorting necessitates the algebraic elimination of $l$ from the graph. Conventional recursive paradigms involve retraining a score model on the marginal distribution $p(\bx_{\setminus l})$. In Gaussian Graphical Models (GGMs) \citep{lauritzen1996graphical, friedman2008sparse}, marginalizing a variable corresponds to taking the Schur complement of the precision matrix. We formalize this property for linear ANMs and score-based SJIMs, showing that recursive model retraining is unnecessary.

\begin{theorem}[Exact Marginalization in Linear ANMs]
\label{thm:schur_marginalization}
Following the marginalization properties of GGMs \citep{lauritzen1996graphical}, let $\I \in \R^{d \times d}$ be the SJIM of the joint distribution $p(\bx)$ under a linear ANM. If $x_l$ is a leaf node, the SJIM of the marginal distribution $p(\bx_{\setminus l})$ is the Schur complement of $\I$ with respect to index $l$:
\begin{equation} \label{eq:schur}
    \I_{\mathrm{marginal}} = \I_{\setminus l} - \I_{\setminus l, l} \left( \I_{ll} \right)^{-1} \I_{l, \setminus l}
\end{equation}
where $\I_{\setminus l}$ is the submatrix excluding the $l$-th row and column. (Formal proof is deferred to Appendix~\ref{app:proofs_schur}).
\end{theorem}

\subsection{The Expectation Gap in Non-Linear ANMs}
For non-linear ANMs, the sample-wise Hessian $H(\bx) = -\nabla_{\bx}^2 \log p(\bx)$ is not constant. Because the expectation operator does not commute with matrix inversion, $\mathrm{Schur}(\E[H(\bx)]) \neq \E[\mathrm{Schur}(H(\bx))]$. We formalize this discrepancy below.

\begin{lemma}[Marginal SJIM Equivalence in Non-linear ANMs]
\label{lemma:marginal_equivalence}
Let $H(\bx) = -\nabla_{\bx}^2 \log p(\bx)$ be the sample-wise Hessian under a non-linear ANM. For a topological leaf node $l$, the expected Schur complement of the joint Hessian is equivalent to the Score-Jacobian Information Matrix of the marginalized distribution $p(\bx_{\setminus l})$: $\E[\mathrm{Schur}(H(\bx))] = \I_{\mathrm{marginal}}$. (Formal proof is deferred to Appendix~\ref{app:proofs_marginal_equivalence}).
\end{lemma}

\begin{proposition}[Non-linear Marginalization Expectation Gap]
\label{prop:nonlinear_error}
Let $H(\bx)$ be the sample-wise Hessian of the log-density under a general non-linear ANM, and let $l$ be a topological leaf node. The difference $\Delta$ between the expectation of the sample-wise Schur complement and the Schur complement of the expected Hessian is exactly the negative scaled covariance matrix of the leaf node's gradient:
\begin{equation}
    \Delta = \E[\mathrm{Schur}(H(\bx))] - \mathrm{Schur}(\E[H(\bx)]) = -\frac{1}{\sigma^2} \Cov_{p(\bx)} \left( \nabla_{\bx_{\setminus l}} f_l(\bx_{\pa(l)}) \right).
\end{equation}
(Formal proof is deferred to Appendix~\ref{app:proofs_expectation_gap}).
\end{proposition}

\begin{remark}[Localization of the Expectation Gap]
\label{rmk:error_localization}
Because the non-linear mechanism $f_l$ depends strictly on its Markovian parents $\pa(l)$, the gradient vector $\nabla_{\bx_{\setminus l}} f_l$ contains non-zero entries exclusively at indices corresponding to $\pa(l)$. Consequently, the covariance matrix $\Cov(\nabla f_l)$ is non-zero solely within the $\pa(l) \times \pa(l)$ submatrix block. This indicates that the approximation error induced by the non-linear expectation gap is localized to the immediate parent neighborhood of the marginalized leaf node. 
\end{remark}
\begin{remark}[Algebraic Closure of the Expectation Gap]
\label{rmk:covariance_patching}
While Block-SSTS compresses the expectation gap cascade, the structural covariance error $\Delta$ can theoretically be resolved analytically. Appendix \ref{app:cov_patching} formulates a closed-form \textit{Covariance Patching} mechanism, utilizing the sample covariance (i.e., second-order central moments) of the off-diagonal blocks to algebraically restore structural covariance. However, evaluating sample-wise Hessian covariances incurs an $\mathcal{O}(N d^2)$ memory footprint and an $\mathcal{O}(N)$ computational overhead per marginalization step. To preserve the $\mathcal{O}(1)$ matrix inversion and $\mathcal{O}(d^2)$ memory efficiency, Block-SSTS omits this patch in practice, trading theoretical fidelity for extreme computational scalability.
\end{remark}
\section{Algorithm: Score-Schur Topological Sort (SSTS)}
\label{sec:algorithm}

SSTS maps continuous marginalization to an empirical algebraic extraction using finite-sample score networks. Combining Theorem~\ref{thm:leaf_discovery} and Theorem~\ref{thm:schur_marginalization} yields a deterministic sorting algorithm. For an approximate score network $s_\theta(\bx)$ trained on observational data, we construct the empirical SJIM using the network's Jacobian to capture structural curvature.

\subsection{High-Dimensional Extension: Sparse Jacobian Prior}
\label{subsec:sparse_prior}
We parameterize $s_\theta(\bx)$ using a Multi-Layer Perceptron (MLP). In high-dimensional regimes ($d \ge 100$), standard score matching often yields an ill-conditioned empirical SJIM. As formalized in Proposition~\ref{prop:nonlinear_error}, uncompensated structural covariance propagates through sequential marginalization. Additionally, finite-sample neural networks introduce estimation variance into the empirical Jacobian $\nabla_{\bx} s_\theta(\bx)$, generating spurious gradient connections.

To control this variance, we introduce a structured sparsity prior. Following the group-wise variable selection principles \citep{yuan2006model}, we apply a Group Lasso ($\ell_{1,2}$) penalty to the input layer weights $W_1 \in \R^{h \times d}$ (where $h$ is the hidden dimensionality) of the MLP:
\begin{equation}
    \mathcal{L}_{\mathrm{sparse}} = \lambda_{\mathrm{sparse}} \sum_{j=1}^d \| [W_1]_{:, j} \|_2
\end{equation}
Penalizing column norms truncates weak input connections prior to non-linear activations. This induces structural sparsity on the global Jacobian estimator, restricting estimation variance in deep DAGs and yielding a well-conditioned, symmetric empirical SJIM: $\hat{\I} = \frac{1}{2}(\hat{J} + \hat{J}^\top)$, where $\hat{J} = \frac{1}{|\mathcal{D}|} \sum_{\bx \in \mathcal{D}} - \nabla_{\bx} s_\theta(\bx)$.

\subsection{Block-Schur Marginalization via Statistical Tolerance}
\label{subsec:block_schur}

In sparse DAGs, multiple variables often occupy the same topological stratum. By Theorem~\ref{thm:leaf_discovery}, parallel leaf nodes yield identical SJIM diagonal energies. In practice, however, finite-sample estimation variance disrupts this exact equivalence in the empirical SJIM.

Applying a strict sequential $\arg\min$ extraction over $\hat{\I}_{ii}$ forces an arbitrary ordering among these parallel nodes based on numerical noise. This spurious sequentialization is structurally unnecessary and artificially triggers the non-linear expectation gap ($\Delta$) formalized in Proposition~\ref{prop:nonlinear_error}, accumulating approximation errors across conditionally independent variables.

To address this, Block-SSTS groups parallel nodes using a relative tolerance parameter $\gamma \in (0, 1)$: $\mathcal{B} = \{ k \in \mathcal{S} \mid \hat{\I}_{kk} \le \min_{i \in \mathcal{S}} (\hat{\I}_{ii}) + \gamma | \min_{i \in \mathcal{S}} (\hat{\I}_{ii}) | \}$. Since the variables in $\mathcal{B}$ are topologically parallel, their theoretical cross-derivatives vanish. This conditional independence allows for a simultaneous block Schur complement inversion. 

This block-wise marginalization natively absorbs empirical estimation noise and reduces the total depth of sequential matrix inversions. Consequently, it truncates the accumulation of $\Delta$, preserving structural fidelity without incurring the $\mathcal{O}(Nd^3)$ computational cost of exact covariance patching (Appendix~\ref{rmk:covariance_patching}). For numerical stability during the empirical block inversion, we apply a Tikhonov regularization term, $\lambda_{\mathrm{ridge}} I$.

\begin{algorithm}[H]
\caption{Structure-Optimization-Free Block-Score-Schur Topological Sort (Block-SSTS)}
\label{alg:ssts}
\begin{algorithmic}[1]
\REQUIRE Pre-trained unconstrained score model $s_\theta(\mathbf{x})$, Observational data $\mathcal{D}$, Ridge penalty $\lambda_{\mathrm{ridge}} > 0$, Block tolerance $\gamma \in (0, 1)$.
\STATE \textbf{Initialize} empirical Jacobian $\hat{J} \leftarrow \frac{1}{|\mathcal{D}|} \sum_{\mathbf{x} \in \mathcal{D}} - \nabla_{\mathbf{x}} s_\theta(\mathbf{x})$.
\STATE \textbf{Symmetrize} to construct a symmetric empirical SJIM $\hat{\mathcal{I}} \leftarrow \frac{1}{2}(\hat{J} + \hat{J}^\top)$.
\STATE \textbf{Initialize} active set $\mathcal{S} \leftarrow \{1, 2, \dots, d\}$, topological order $\pi \leftarrow [ \ ]$.
\WHILE{$|\mathcal{S}| > 0$}
    \STATE Isolate leaf block $\mathcal{B}$ using energy tolerance.
    \STATE Sort nodes within $\mathcal{B}$ by diagonal energy and prepend to the ordering list $\pi$.
    \IF{$|\mathcal{S} \setminus \mathcal{B}| > 0$}
        \STATE Let $\mathcal{S}' = \mathcal{S} \setminus \mathcal{B}$. Compute Tikhonov-regularized Block Schur complement:
        \STATE $\hat{\mathcal{I}}_{\mathcal{S}', \mathcal{S}'} \leftarrow \hat{\mathcal{I}}_{\mathcal{S}', \mathcal{S}'} - \hat{\mathcal{I}}_{\mathcal{S}', \mathcal{B}} (\hat{\mathcal{I}}_{\mathcal{B}, \mathcal{B}} + \lambda_{\mathrm{ridge}} I)^{-1} \hat{\mathcal{I}}_{\mathcal{B}, \mathcal{S}'}$
    \ENDIF
    \STATE Update the active set: $\mathcal{S} \leftarrow \mathcal{S} \setminus \mathcal{B}$.
\ENDWHILE
\RETURN Topological order $\pi$.
\end{algorithmic}
\end{algorithm}

Once the topological order $\pi$ is extracted, resolving the global DAG reduces to an independent feature selection problem for each node over its structurally permitted predecessors. Following standard order-based pruning paradigms \citep{buhlmann2014cam}, this decomposes the acyclicity constraint into $d$ localized penalized regression tasks:
\[
    \min_{f_i} \E \left[ (x_i - f_i(\bx_{\mathrm{pre}_{\pi}(i)}))^2 \right] + \Omega(f_i)
\]
where $\mathrm{pre}_{\pi}(i)$ denotes candidate ancestors preceding node $i$, and $\Omega(f_i)$ imposes structural sparsity.

\textbf{Computational Complexity.} 
The primary computational cost lies in the empirical estimation of the SJIM. Constructing $\hat{\I}$ requires evaluating the Jacobian of the neural score network. Unlike continuous constrained optimization methods that evaluate $\mathcal{O}(d^3)$ acyclicity penalties repeatedly per gradient step, SSTS isolates structural extraction to a single pass of Jacobian extraction. Once $\hat{\I}$ is materialized, the global DAG is resolved via standard matrix operations.

\section{Experiments}
\label{sec:experiments}

We evaluate SSTS to verify its algebraic equivalence under linear conditions, analyze its performance on non-linear manifolds, and establish its statistical boundaries. Computations are executed on a single NVIDIA RTX 5090 GPU.

\textbf{Implementation Details.} 
SSTS avoids the $\mathcal{O}(d^3)$ acyclicity penalty $h(W)$ during training; the score network functions as an unconstrained density estimator parameterized by a multi-layer perceptron. \textbf{SSTS (Vanilla)} denotes extraction without Jacobian sparsity priors ($\lambda_{\mathrm{sparse}} = 0$) for low-dimensional tasks ($d \le 20$). For $d \ge 50$, \textbf{SSTS (Sparse)} applies a Group Lasso prior scaled by $\lambda_{\mathrm{sparse}} \propto \sqrt{d}$, with a constant ridge penalty $\lambda_{\mathrm{ridge}} = 10^{-4}$ for matrix inversion stability. To construct the empirical SJIM $\hat{\mathcal{I}}$, we compute the full $d \times d$ sample-wise Jacobians of the score network via vectorized automatic differentiation (`vmap`). Instead of storing $N$ matrices, we accumulate the expected Hessian via streaming mini-batches, bounding the memory footprint strictly to $\mathcal{O}(d^2)$ (e.g., $\sim 4$ MB for $d=1000$ in 32-bit float). We use 32-bit floating point for SJIM construction and extraction. Execution time is decoupled into generative representation learning ($T_{\mathrm{rep}}$), algebraic structural extraction ($T_{\mathrm{disc}}$), and downstream post-hoc edge pruning ($T_{\mathrm{prune}}$). To ensure algorithmic decoupling is evaluated without confounding software optimization variables, structural baselines are executed within a strictly controlled capacity environment (detailed in Appendix \ref{app:baseline_implementations}).

\textbf{Evaluation Metrics and Topological Ambiguity.} Structural accuracy is evaluated via Structural Hamming Distance (SHD), True Positive Rate (TPR), and Edge Violations (EV). Here, EV is rigorously defined as the number of true causal edges $j \to i \in \mathcal{E}$ that contradict the extracted topological order (i.e., where $\pi(j) > \pi(i)$). Sequence correlation metrics (e.g., Kendall's $\tau$) are strictly omitted. As we formally prove in Appendix~\ref{app:metric_justification}, evaluating a highly parallel DAG against a single ground-truth array via 1D rank correlation is mathematically ill-posed; the expected correlation degrades quadratically with the width of parallel strata, even for perfectly extracted causal geometries. Consequently, rigorous evaluation necessitates a metric strictly invariant to topological ambiguity, rendering EV the exact objective measure.

\subsection{Verification of Exact Recovery (Linear ANMs)}
\label{subsec:linear_exact}
Evaluations on linear ANMs confirm zero causal edge violations across all dimensions up to $d=1000$ for both population and empirical precision matrices, empirically validating Theorem~\ref{thm:schur_marginalization} independently of neural network estimation variance (numerical details in Appendix~\ref{app:linear_exact}).
\subsection{Baseline Comparisons: Algorithmic Decoupling on Non-Linear Manifolds}
\label{subsec:nonlinear_benchmark}

\begin{figure}[ht]
    \centering
    \includegraphics[width=\textwidth]{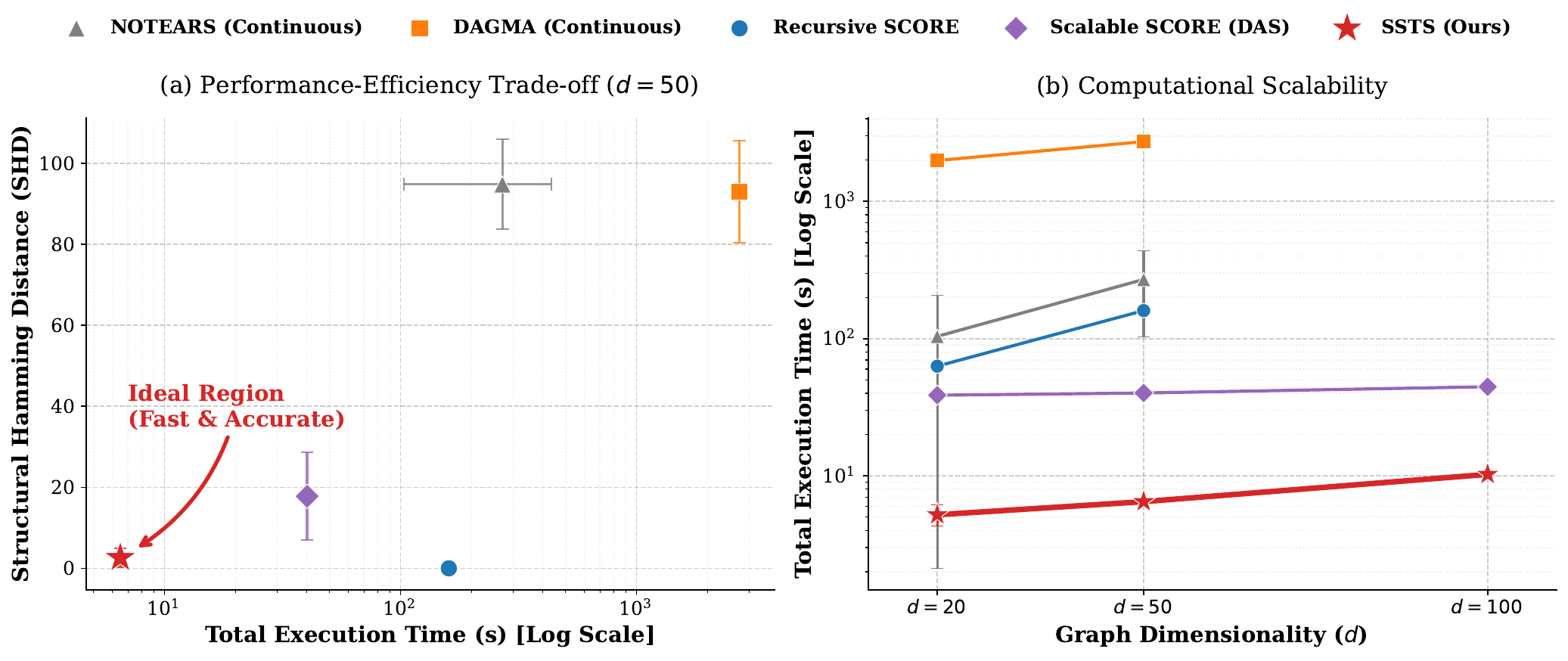}
    \caption{\textbf{Algorithmic Decoupling on Non-Linear Manifolds ($N=5000$).} 
    \textbf{(a)} The performance-efficiency trade-off evaluated at $d=50$. SSTS recovers graph structure efficiently compared to continuous optimization techniques. 
    \textbf{(b)} Computational scalability across dimensions. The gradient-free algebraic extraction of SSTS requires less than 15 seconds of total execution time at $d=100$.}
    \label{fig:main_benchmark}
\end{figure}

To evaluate the algorithm on non-linear manifolds, we formulate structural equations using $\tanh$ mechanisms ($N=5000$). Table~\ref{tab:main_benchmark_app} reports the empirical performance against continuous constrained optimizers (DAGMA, NOTEARS) and score-based paradigms (Recursive SCORE, Scalable SCORE). For a rigorous comparison, the total algorithmic time ($T_{\mathrm{algo}}$) encapsulates both representation learning ($T_{\mathrm{rep}}$) and structural extraction ($T_{\mathrm{disc}}$). Since continuous optimizers learn structure jointly, their execution time maps directly to $T_{\mathrm{algo}}$.

\textbf{Computational Bottleneck of Continuous Optimization.} Continuous constrained methods struggle to scale on non-linear manifolds. At $d=50$, DAGMA requires over 2700 seconds to evaluate the non-convex augmented Lagrangian, yielding an SHD of $93.0 \pm 12.6$. By eliminating the iterative acyclicity penalty, SSTS achieves $\text{SHD} = 2.6 \pm 2.4$ directly from the pre-trained SJIM, reducing $T_{\mathrm{algo}}$ to 3.66 seconds.

\textbf{The Computation-Estimation Trade-off.} At $d=100$, SSTS performs marginalization in a single algebraic pass ($T_{\mathrm{disc}} = 0.04$s), but exhibits higher structural error compared to dynamic masked evaluations ($\text{SHD} = 64.6 \pm 17.3$ vs. $10.4 \pm 5.5$ for Scalable SCORE). This performance divergence is directly attributable to the uncompensated non-linear expectation gap ($\Delta$) derived in Proposition~\ref{prop:nonlinear_error}. 

\begin{table}[ht]
\centering
\caption{\textbf{Evaluation of the Non-linear Expectation Gap ($\Delta$).} Exact sample-wise Schur marginalization analytically eliminates $\Delta$, isolating algorithmic extraction error from statistical estimation limits. Evaluated at $N=5000$.}
\label{tab:exact_schur_tradeoff}
\resizebox{0.95\textwidth}{!}{
\begin{tabular}{lcccc}
\toprule
& \multicolumn{2}{c}{\textbf{Standard Block-SSTS (Fast, $\Delta$ Uncompensated)}} & \multicolumn{2}{c}{\textbf{Exact Sample-wise SSTS (Slow, No $\Delta$)}} \\
\cmidrule(lr){2-3} \cmidrule(lr){4-5}
\textbf{Dimension ($d$)} & \textbf{Edge Violations} $\downarrow$ & \textbf{Extraction Time ($T_{\mathrm{disc}}$)} $\downarrow$ & \textbf{Edge Violations} $\downarrow$ & \textbf{Extraction Time ($T_{\mathrm{disc}}$)} $\downarrow$ \\
\midrule
50   & 2   & 0.23s & \textbf{1}   & 0.27s \\
100  & 38  & \textbf{0.04s} & \textbf{36}  & 4.47s \\
500  & 221 & \textbf{0.26s} & \textbf{176} & 574.61s \\
1000 & 444 & \textbf{0.63s} & \textbf{300} & 4493.19s \\
\bottomrule
\end{tabular}
}
\end{table}

Table~\ref{tab:exact_schur_tradeoff} shows that computing the exact sample-wise Schur complement prior to the expectation operator analytically eliminates $\Delta$, strictly reducing Edge Violations. However, evaluating sample-wise Hessian covariances imposes an $\mathcal{O}(N d^3)$ computational overhead. At $d=1000$, this exact formulation increases extraction time by a factor of $\sim 7000\times$. Furthermore, the residual error ($300$ violations) under exact algebraic marginalization confirms that structural fidelity is ultimately bounded not by the algebraic approximation, but by the estimation variance of the score Jacobian under finite-sample statistical starvation ($N=5000$). Block-SSTS omits the covariance patch to preserve $\mathcal{O}(1)$ block inversions, transferring the causal discovery bottleneck directly to this statistical estimation limit.

\subsection{Application to Real-World Biological Data}
\label{subsec:real_world}

To evaluate robustness against practical distribution shifts and latent confounding, we benchmark SSTS on the continuous Sachs protein-signaling network \citep{sachs2005causal} ($d=12, N=5000$). Structural extraction on this dataset is challenged by violations of strict additive noise assumptions, typically resulting in elevated Structural Hamming Distances (SHD) across continuous methods. Despite these deviations, SSTS achieves structural recovery yielding an SHD of $13.8 \pm 1.2$ and a True Positive Rate (TPR) of $0.40 \pm 0.04$ (compared to DAGMA's SHD of $15.0 \pm 1.7$ and Scalable SCORE's SHD of $16.2 \pm 1.5$). By decoupling the topological sort from generative modeling, SSTS completes the extraction in $6.62$ seconds, circumventing the computational overhead of non-convex optimization (DAGMA at $182.36$s) and masked Jacobian queries (Scalable SCORE at $56.32$s). Comprehensive tabular results are deferred to Appendix~\ref{app:sachs_table}.

\subsection{Ablation Studies: Mechanisms under Extreme Regimes}
\label{subsec:ablation}

\begin{figure}[ht]
    \centering
    \includegraphics[width=\textwidth]{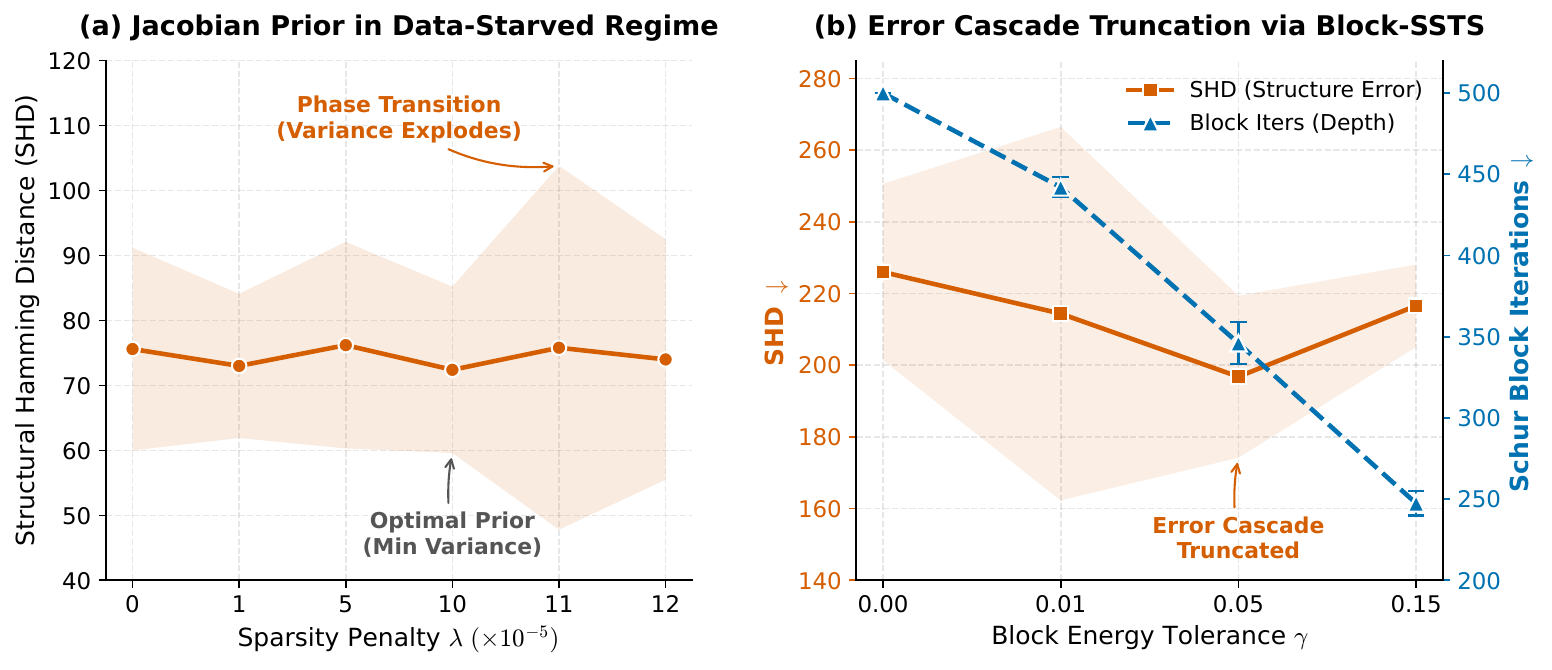}
    \caption{\textbf{Dual-Regime Ablation Studies.} 
    \textbf{(a)} In the data-starved regime ($d=100, N=500$), the Group Lasso penalty ($\lambda$) suppresses estimation variance. Exceeding the optimal prior triggers a transition where structural variance spikes. 
    \textbf{(b)} In the extreme scaling regime ($d=500, N=5000$), increasing the block tolerance ($\gamma$) compresses the matrix inversion depth (blue dashed line). This physical compression truncates the non-linear error cascade, tracking SHD trajectories (orange solid line).}
    \label{fig:ablation}
\end{figure}

We design two boundary-condition experiments: a data-starved regime to evaluate the sparse Jacobian prior, and an extreme scaling regime to evaluate block marginalization.

\textbf{Jacobian Sparsity in Data-Starved Regimes.} 
When the sample-to-feature ratio is limited ($N=500, d=100$), unconstrained neural networks exhibit estimation variance. As shown in Figure~\ref{fig:ablation}a, the baseline model ($\lambda = 0.0$) yields high cross-seed structural variance. Introducing the Group Lasso prior ($\lambda = 10 \times 10^{-5}$) mitigates finite-sample overfitting, achieving the minimum SHD and narrowing the error band. Increasing the penalty marginally ($\lambda = 11 \times 10^{-5}$) triggers a sharp variance spike. This indicates the boundary where the structural penalty suppresses true causal gradients.

\textbf{Block-SSTS Tolerance in Extreme Scaling.}
To quantify the structural error accumulation formulated in Proposition~\ref{prop:nonlinear_error}, we scale the topology to $d=500$. Figure~\ref{fig:ablation}b illustrates the effect of the block tolerance $\gamma$. Sequential marginalization ($\gamma=0.00$) forces $500$ Schur inversions, accumulating the non-linear expectation gap. Applying a moderate tolerance ($\gamma=0.05$) compresses the extraction depth to $346$ iterations. This truncation of the error cascade reduces the SHD to its minimum. An aggressive tolerance ($\gamma=0.15$) forces topologically distinct nodes into parallel blocks, violating the causal hierarchy and increasing structural error. Tabular data is provided in Appendix~\ref{app:ablation_tables}.

\subsection{Sensitivity to the Additive Noise Assumption}
\label{subsec:anm_violation}

The theoretical identifiability established in Theorem~\ref{thm:leaf_discovery} relies on the Additive Noise Model (ANM), where independent noise isolates the topological signature on the SJIM diagonal. We conduct a failure mode analysis ($d=50, N=5000$) by violating this assumption. 

Against an ANM control group (SHD $3.6 \pm 2.9$, TPR $0.95 \pm 0.04$), we evaluate Multiplicative Noise Models (MNM), formulated as $x_i = f_i(\bx_{\pa(i)}) + g_i(\bx_{\pa(i)}) \epsilon_i$. SSTS remains robust to this conditional heteroscedasticity (SHD $0.4 \pm 0.8$, TPR $0.99 \pm 0.02$). Parental dependency in the noise scale introduces a distributional asymmetry that enhances causal identifiability within the unconstrained score geometry.

Under Post-Nonlinear (PNL) mechanisms \citep{zhang2009identifiability}, defined as $x_i = (f_i(\bx_{\pa(i)}) + \epsilon_i)^3$, the algebraic extraction fails (SHD $94.4 \pm 8.4$, TPR $0.03 \pm 0.02$). The outer non-linear mapping couples the noise components with the structural parents in the gradient space, violating the statistical independence required for vanishing cross-derivatives. Extending exact Schur marginalization to arbitrary PNL transformations thus remains an open problem. Tabular results are deferred to Appendix~\ref{app:anm_table}.

\section{Discussion and Limitations}
\label{sec:discussion}

SSTS yields an exact equivalence in linear Gaussian ANMs. 
In non-linear ANMs, the algebraic extraction introduces an expectation gap and exposes finite-sample estimation limits of the score Jacobian.

\textbf{Topological Ambiguity vs. Algebraic Truth.} Traditional metrics like Kendall's $\tau$ evaluate the distance between a predicted sequence and a single ground-truth generative order. However, disconnected or parallel structures in sparse DAGs permit multiple valid topological sorts. As formalized by the degradation limit in Theorem~\ref{thm:kendall_collapse} (Appendix~\ref{app:metric_justification}), rank correlation collapses mathematically under parallel graph structures, penalizing geometrically identical permutations. Evaluating causal Edge Violations provides the only direct and unambiguous measure, ensuring that algorithmic extraction is evaluated independently of inherent topological ambiguity.

\textbf{The Non-linear Expectation Gap.} As quantified in Proposition~\ref{prop:nonlinear_error}, the Schur complement of the expected Hessian incurs an approximation gap proportional to $\Cov(\nabla f_l)$. In linear systems, $\nabla f_l$ is constant, the covariance vanishes, and extraction is exact. In non-linear systems, this gap is non-zero. Over $\mathcal{O}(d)$ sequential marginalization steps, uncompensated structural covariance accumulates, degrading the SJIM's topological signature.

\textbf{The Computation-Estimation Trade-off.} Standard continuous causal discovery minimizes a data-fidelity loss alongside a non-convex DAG penalty $h(W)=0$, decomposing the extraction error into optimization and statistical gaps ($\mathcal{E}_{\text{total}} \approx \mathcal{E}_{\text{opt}} + \mathcal{E}_{\text{stat}}$). Gradient-based solvers stall in local minima within high-dimensional spaces, causing $\mathcal{E}_{\text{opt}}$ to dominate the error profile alongside an $\mathcal{O}(d^3)$ computational requirement. By mapping graph marginalization to the Schur complement, SSTS algebraically reduces the topological optimization gap for structural extraction to zero ($\mathcal{E}_{\text{opt}} = 0$). The complexity of causal discovery is conserved and transferred to the statistical domain ($\mathcal{E}_{\text{stat}}$). Consequently, the structural fidelity of SSTS is bounded by two factors: (1) \textit{SJIM estimation variance}, as finite-sample neural networks struggle to yield unbiased precision geometry in extreme high-dimensional regimes (see Appendix~\ref{app:extreme_scaling}); and (2) \textit{non-linear error accumulation}, governed by the structural covariance gap ($\Delta$) formalized in Proposition~\ref{prop:nonlinear_error}. While SSTS circumvents the non-convex optimization bottleneck, operational limits are dictated by the fidelity of the generative score estimation.

\section{Conclusion}
\label{sec:conclusion}
We introduced the Score-Schur Topological Sort (SSTS), which relates leaf-node marginalization to Schur-complement-based elimination on the Score-Jacobian Information Matrix (SJIM). 
The mapping is exact in linear Gaussian ANMs and approximate in non-linear ANMs with a quantified expectation gap.

Decoupling representation learning from structural extraction shifts the fundamental bottleneck of causal discovery from combinatorial optimization to high-dimensional statistical estimation. The structural fidelity of causal discovery is no longer constrained by the local optima of acyclicity penalties, but by the capacity of generative models to estimate exact score functions.

\bibliographystyle{plainnat}
\bibliography{reference}

\begin{appendix}
\section{Deferred Proofs and Derivations}
\label{app:proofs}

\subsection{Proof of Theorem \ref{thm:leaf_discovery}}
\label{app:proofs_leaf}
\begin{proof}
By the Markov property of the ANM, the joint distribution factorizes as $p(\bx) = \prod_{k=1}^d p_{\epsilon_k}(x_k - f_k(\bx_{\pa(k)}))$. The joint log-likelihood is:
\begin{equation}
    \log p(\bx) = \sum_{k=1}^d \log p_{\epsilon_k}(x_k - f_k(\bx_{\pa(k)})).
\end{equation}
Applying the chain rule, the first partial derivative with respect to $x_i$ is:
\begin{equation}
    \frac{\partial \log p(\bx)}{\partial x_i} = \frac{\partial \log p_{\epsilon_i}(\epsilon_i)}{\partial \epsilon_i} \frac{\partial \epsilon_i}{\partial x_i} + \sum_{j \in \ch(i)} \frac{\partial \log p_{\epsilon_j}(\epsilon_j)}{\partial \epsilon_j} \frac{\partial \epsilon_j}{\partial x_i}.
\end{equation}
Since $\epsilon_k = x_k - f_k(\bx_{\pa(k)})$, we have $\frac{\partial \epsilon_i}{\partial x_i} = 1$ and $\frac{\partial \epsilon_j}{\partial x_i} = -\frac{\partial f_j}{\partial x_i}$. Given Gaussian noise, $\frac{\partial \log p_{\epsilon_k}(\epsilon_k)}{\partial \epsilon_k} = -\frac{\epsilon_k}{\sigma^2}$. Substituting these yields the score component:
\begin{equation}
    \frac{\partial \log p(\bx)}{\partial x_i} = -\frac{\epsilon_i}{\sigma^2} + \sum_{j \in \ch(i)} \frac{\epsilon_j}{\sigma^2} \frac{\partial f_j}{\partial x_i}.
\end{equation}
Differentiating again with respect to $x_i$ to obtain the diagonal of the Hessian:
\begin{equation} \label{eq:app_hessian_diag}
    \frac{\partial^2 \log p(\bx)}{\partial x_i^2} = -\frac{1}{\sigma^2} + \sum_{j \in \ch(i)} \left( \frac{1}{\sigma^2} \frac{\partial \epsilon_j}{\partial x_i} \frac{\partial f_j}{\partial x_i} + \frac{\epsilon_j}{\sigma^2} \frac{\partial^2 f_j}{\partial x_i^2} \right).
\end{equation}
Substituting $\frac{\partial \epsilon_j}{\partial x_i} = -\frac{\partial f_j}{\partial x_i}$ into Eq. \ref{eq:app_hessian_diag} yields:
\begin{equation}
    \frac{\partial^2 \log p(\bx)}{\partial x_i^2} = -\frac{1}{\sigma^2} - \frac{1}{\sigma^2} \sum_{j \in \ch(i)} \left( \frac{\partial f_j}{\partial x_i} \right)^2 + \sum_{j \in \ch(i)} \frac{\epsilon_j}{\sigma^2} \frac{\partial^2 f_j}{\partial x_i^2}.
\end{equation}
To compute the Score-Jacobian Information Matrix diagonal $\I_{ii} = \E[-\frac{\partial^2 \log p(\bx)}{\partial x_i^2}]$, we take the negative expectation over $p(\bx)$. By the definition of ANM, the noise term $\epsilon_j$ is statistically independent of its causal ancestors $\bx_{\pa(j)}$. Since $\frac{\partial^2 f_j}{\partial x_i^2}$ depends on $\bx_{\pa(j)}$, we have:
\begin{equation}
    \E_{p(\bx)}\left[ \epsilon_j \frac{\partial^2 f_j}{\partial x_i^2} \right] = \E[\epsilon_j] \E\left[ \frac{\partial^2 f_j}{\partial x_i^2} \right] = 0 \cdot \E\left[ \frac{\partial^2 f_j}{\partial x_i^2} \right] = 0.
\end{equation}
The expectation of the last term vanishes, leaving:
\begin{equation}
    \I_{ii} = \frac{1}{\sigma^2} + \frac{1}{\sigma^2} \sum_{j \in \ch(i)} \E_{p(\bx)} \left[ \left( \frac{\partial f_j}{\partial x_i} \right)^2 \right].
\end{equation}
Because $\E[(\partial f_j / \partial x_i)^2] > 0$ for any true topological parent-child relationship, it follows that $\I_{ii} > 1/\sigma^2$ for any node with children. For a leaf node $l$, $\ch(l) = \emptyset$, resulting exactly in $\I_{ll} = 1/\sigma^2$.
\end{proof}

\subsection{Proof of Theorem \ref{thm:schur_marginalization}}
\label{app:proofs_schur}
\begin{proof}
Assume the variables are permuted such that the leaf node $l$ corresponds to the last row and column of the matrix. For a linear ANM, Eq. \ref{eq:anm} takes the matrix form $\bx = B \bx + \boldsymbol{\epsilon}$, where $B$ is a strictly upper-triangular weighted adjacency matrix. The explicit solution is $\bx = (I - B)^{-1} \boldsymbol{\epsilon}$. The covariance matrix is $\Sigma = (I - B)^{-1} (\sigma^2 I) (I - B)^{-\top}$, and the precision matrix (which equals the SJIM for Gaussian distributions) is:
\begin{equation}
    \I = \Sigma^{-1} = \frac{1}{\sigma^2} (I - B)^\top (I - B).
\end{equation}
Let $\mathcal{S} = \mathcal{V} \setminus \{l\}$. Since $l$ is a leaf node, it has no children, implying the $l$-th column of $B$ is entirely zero: $B_{\mathcal{S}, l} = \mathbf{0}$. We partition $(I - B)$ into block form:
\begin{equation}
    I - B = \begin{pmatrix} I_{\mathcal{S}} - B_{\mathcal{S}\mathcal{S}} & \mathbf{0} \\ -B_{l, \mathcal{S}} & 1 \end{pmatrix}.
\end{equation}
Multiplying this block matrix by its transpose, we obtain the partitioned SJIM:
\begin{equation} \label{eq:app_partitioned_I}
    \I = \frac{1}{\sigma^2} \begin{pmatrix} (I_{\mathcal{S}} - B_{\mathcal{S}\mathcal{S}})^\top (I_{\mathcal{S}} - B_{\mathcal{S}\mathcal{S}}) + B_{l, \mathcal{S}}^\top B_{l, \mathcal{S}} & -B_{l, \mathcal{S}}^\top \\ -B_{l, \mathcal{S}} & 1 \end{pmatrix}.
\end{equation}
From Eq. \ref{eq:app_partitioned_I}, we explicitly extract the blocks: $\I_{ll} = \frac{1}{\sigma^2}$, $\I_{l, \mathcal{S}} = -\frac{1}{\sigma^2} B_{l, \mathcal{S}}$, $\I_{\mathcal{S}, l} = -\frac{1}{\sigma^2} B_{l, \mathcal{S}}^\top$, and $\I_{\mathcal{S}\mathcal{S}} = \frac{1}{\sigma^2} (I_{\mathcal{S}} - B_{\mathcal{S}\mathcal{S}})^\top (I_{\mathcal{S}} - B_{\mathcal{S}\mathcal{S}}) + \frac{1}{\sigma^2} B_{l, \mathcal{S}}^\top B_{l, \mathcal{S}}$. 

Computing the Schur complement of $\I$ with respect to $\I_{ll}$ yields:
\begin{align}
    \mathrm{Schur}(\I) &= \I_{\mathcal{S}\mathcal{S}} - \I_{\mathcal{S}, l} \I_{ll}^{-1} \I_{l, \mathcal{S}} \\
    &= \I_{\mathcal{S}\mathcal{S}} - \left( -\frac{1}{\sigma^2} B_{l, \mathcal{S}}^\top \right) (\sigma^2) \left( -\frac{1}{\sigma^2} B_{l, \mathcal{S}} \right) \\
    &= \I_{\mathcal{S}\mathcal{S}} - \frac{1}{\sigma^2} B_{l, \mathcal{S}}^\top B_{l, \mathcal{S}}.
\end{align}
Substituting the expression for $\I_{\mathcal{S}\mathcal{S}}$, the cross-terms cancel out:
\begin{equation}
    \mathrm{Schur}(\I) = \frac{1}{\sigma^2} (I_{\mathcal{S}} - B_{\mathcal{S}\mathcal{S}})^\top (I_{\mathcal{S}} - B_{\mathcal{S}\mathcal{S}}).
\end{equation}
This evaluates to the precision matrix of the marginal distribution $p(\bx_{\setminus l})$. The Schur complement marginalizes the leaf node while preserving the causal structural matrix $B_{\mathcal{S}\mathcal{S}}$ of the remaining subgraph.
\end{proof}

\subsection{Proof of Lemma \ref{lemma:marginal_equivalence}}
\label{app:proofs_marginal_equivalence}
\begin{proof}
By the definition of conditional probability, the joint log-likelihood decomposes as $\log p(\bx) = \log p(\bx_{\setminus l}) + \log p_{\epsilon_l}(x_l - f_l(\bx_{\pa(l)}))$. Differentiating twice with respect to the remaining variables $\bx_{\setminus l}$ yields:
\begin{equation}
    \nabla_{\bx_{\setminus l}}^2 \log p(\bx) = \nabla_{\bx_{\setminus l}}^2 \log p(\bx_{\setminus l}) - \frac{1}{\sigma^2} \nabla f_l \nabla f_l^\top + \frac{x_l - f_l}{\sigma^2} \nabla^2 f_l.
\end{equation}
Taking the negative expectation over $p(\bx)$ produces the upper-left block of the joint SJIM, $\I_{\setminus l, \setminus l}$. Since the noise $\epsilon_l = x_l - f_l$ is zero-mean and independent of the ancestors $\bx_{\setminus l}$, the expectation of the cross-term vanishes ($\E[\epsilon_l \nabla^2 f_l] = 0$). Thus:
\begin{equation}
    \I_{\setminus l, \setminus l} = \I_{\mathrm{marginal}} + \frac{1}{\sigma^2} \E[\nabla f_l \nabla f_l^\top].
\end{equation}
Rearranging this equation and evaluating $\E[\mathrm{Schur}(H(\bx))] = \I_{\setminus l, \setminus l} - \frac{1}{\sigma^2} \E[\nabla f_l \nabla f_l^\top]$, we recover $\I_{\mathrm{marginal}} = \E[\mathrm{Schur}(H(\bx))]$.
\end{proof}

\subsection{Proof of Proposition \ref{prop:nonlinear_error}}
\label{app:proofs_expectation_gap}
\begin{proof}
For a leaf node $l$, the second derivative $\frac{\partial^2 \log p}{\partial x_l^2} = -\frac{1}{\sigma^2}$ everywhere. The sample-wise block $H_{ll}(\bx) = \frac{1}{\sigma^2}$ is a constant scalar. The cross-derivative for $i \in \pa(l)$ is:
\begin{equation}
    \frac{\partial^2 \log p(\bx)}{\partial x_i \partial x_l} = \frac{1}{\sigma^2} \frac{\partial f_l}{\partial x_i}.
\end{equation}
The off-diagonal block evaluates to $H_{l, \setminus l}(\bx) = -\frac{1}{\sigma^2} \nabla_{\bx_{\setminus l}} f_l(\bx)^\top$.
Evaluating the expectation of the sample-wise Schur complement:
\begin{align}
    \E[\mathrm{Schur}(H(\bx))] &= \E\left[ H_{\setminus l, \setminus l}(\bx) - H_{\setminus l, l}(\bx) H_{ll}(\bx)^{-1} H_{l, \setminus l}(\bx) \right] \\
    &= \I_{\setminus l, \setminus l} - \frac{1}{\sigma^2} \E\left[ \nabla f_l \nabla f_l^\top \right].
\end{align}
Conversely, the Schur complement of the global expectation $\I = \E[H(\bx)]$ yields:
\begin{align}
    \mathrm{Schur}(\E[H(\bx)]) &= \I_{\setminus l, \setminus l} - \I_{\setminus l, l} \I_{ll}^{-1} \I_{l, \setminus l} \\
    &= \I_{\setminus l, \setminus l} - \frac{1}{\sigma^2} \E[\nabla f_l] \E[\nabla f_l]^\top.
\end{align}
Taking the difference between the two quantities, $\I_{\setminus l, \setminus l}$ cancels out:
\begin{align}
    \Delta &= \left( \I_{\setminus l, \setminus l} - \frac{1}{\sigma^2} \E[\nabla f_l \nabla f_l^\top] \right) - \left( \I_{\setminus l, \setminus l} - \frac{1}{\sigma^2} \E[\nabla f_l] \E[\nabla f_l]^\top \right) \\
    &= -\frac{1}{\sigma^2} \left( \E[\nabla f_l \nabla f_l^\top] - \E[\nabla f_l] \E[\nabla f_l]^\top \right) \\
    &= -\frac{1}{\sigma^2} \Cov_{p(\bx)} \left( \nabla f_l(\bx) \right).
\end{align}
\end{proof}
\section{Theoretical Analysis of Evaluation Metrics: The Degradation of Kendall's \texorpdfstring{$\tau$}{tau}}
\label{app:metric_justification}

In continuous causal discovery, the structural fidelity of an extracted sequence $\hat{\pi}$ is frequently evaluated against a ground-truth topological order $\pi^*$ using 1D rank correlation metrics, such as Kendall's $\tau$. We demonstrate that for sparse DAGs characterized by parallel branches and wide leaf strata, rank correlation is mathematically ill-posed.

Let $\mathcal{G} = (\mathcal{V}, \mathcal{E})$ be a DAG with $d$ variables. A topological sort maps $\mathcal{G}$ to a 1D sequence such that for every directed edge $j \to i \in \mathcal{E}$, node $j$ precedes node $i$. Let $\Pi(\mathcal{G})$ denote the \textit{topological equivalence class} containing all valid permutations that satisfy $\mathcal{E}$. 

\begin{definition}[Topological Ambiguity]
Two nodes $u, v \in \mathcal{V}$ are topologically parallel (i.e., conditionally independent given their ancestors without any directed path between them) if both permutations $\dots, u, \dots, v, \dots$ and $\dots, v, \dots, u, \dots$ result in valid topological sorts within $\Pi(\mathcal{G})$. 
\end{definition}

Consequently, forcing a highly parallel DAG into a single ground-truth array $\pi^*$ imposes an arbitrary strict ordering on variables that share no causal dependency. Evaluating a predicted sequence $\hat{\pi}$ against $\pi^*$ via rank correlation penalizes the algorithm for structural ambiguity rather than causal falsity. We formalize this geometric degradation below.

\begin{theorem}[Degradation Limit of Kendall's $\tau$]
\label{thm:kendall_collapse}
Let $\mathcal{G}$ be a DAG consisting of a causal ancestral structure and a stratum of $W$ parallel leaf nodes. Let $\pi^*, \hat{\pi} \in \Pi(\mathcal{G})$ be two perfectly valid topological sorts uniformly sampled from the equivalence class. Although both sequences capture the exact causal graph geometry (zero edge violations), the expected Kendall rank correlation between them degrades quadratically with the width $W$:
\begin{equation}
    \E[\tau(\pi^*, \hat{\pi})] = 1 - \frac{W(W-1)}{d(d-1)}
\end{equation}
Consequently, as the DAG broadens and the proportion of parallel nodes increases ($W \to d$), the expected rank correlation collapses to zero: $\lim_{W \to d} \E[\tau] = 0$.
\end{theorem}
\begin{proof}
Kendall's $\tau$ computes $\tau = 1 - \frac{2K}{d(d-1)/2}$, where $K$ is the number of discordant pairs (inversions) between $\pi^*$ and $\hat{\pi}$. Because both $\pi^*$ and $\hat{\pi}$ are valid sorts within $\Pi(\mathcal{G})$, no true causal edges in $\mathcal{E}$ are inverted. Inversions strictly occur among the $W$ parallel nodes. Since $\hat{\pi}$ and $\pi^*$ independently permute these $W$ nodes, the probability of any pair being discordant is $1/2$. The expected number of discordant pairs is $\E[K] = \frac{1}{2} \binom{W}{2} = \frac{W(W-1)}{4}$. Substituting $\E[K]$ into the definition of $\tau$ yields $1 - \frac{W(W-1)}{d(d-1)}$.
\end{proof}

\begin{remark}[The Necessity of Edge Violations]
Theorem~\ref{thm:kendall_collapse} establishes that 1D rank correlation does not measure causal geometric fidelity; rather, it measures the arbitrary distance to a 1D projection of unidentifiable parallel strata. Evaluating causal extraction necessitates a metric strictly invariant to this topological ambiguity. Therefore, we utilize \textbf{Edge Violations (EV)}, strictly defined as the explicit contradiction of the invariant causal geometry:
\begin{equation}
    \mathrm{EV}(\hat{\pi}, \mathcal{G}) = \big| \{ (j, i) \in \mathcal{E} \mid \hat{\pi}(j) > \hat{\pi}(i) \} \big|
\end{equation}
While EV accurately quantifies the directional integrity of the extracted topological order $\hat{\pi}$, evaluating the exact sparsity of the equivalence class $\Pi(\mathcal{G})$ is deferred to the subsequent structural pruning phase (evaluated via SHD).
\end{remark}
\section{Detailed Results for Linear Exact Recovery}
\label{app:linear_exact}

Table~\ref{tab:linear_exact} provides the exact recovery evaluations on linear ANMs ($N=10000$). We evaluate the topological extraction using the empirical precision matrix inverted from observational data, and the true population precision matrix analytically derived from structural equations.

\begin{table}[ht]
\centering
\caption{\textbf{Exact Recovery on Linear ANMs.} Edge violations recorded over $\mathcal{O}(d^3)$ Schur marginalization steps. The population matrix serves as analytical validation of Theorem~\ref{thm:schur_marginalization}.}
\label{tab:linear_exact}
\resizebox{0.8\textwidth}{!}{
\begin{tabular}{rcc}
\toprule
\textbf{Dimension ($d$)} & \textbf{Empirical Matrix Violations} & \textbf{Population Matrix Violations} \\
\midrule
10  & 0 & 0 (Exact) \\
20  & 0 & 0 (Exact) \\
50  & 0 & 0 (Exact) \\
100 & 0 & 0 (Exact) \\
500 & 0 & 0 (Exact) \\
1000 & 0 & 0 (Exact) \\
\bottomrule
\end{tabular}
}
\end{table}

The experiments yield zero causal edge violations across dimensions ($d \in [10, 1000]$). Zero violations on the population matrix empirically verify Theorem~\ref{thm:schur_marginalization}. Performance on the empirical matrix demonstrates that linear estimation remains structurally faithful at high dimensions.

\section{Theoretical Frontiers: Boundaries of Exact Marginalization}
\label{app:theoretical_frontiers}

\subsection{Scale-Invariant Extraction via Relative Diagonal Variance}
\label{app:heteroscedasticity}

Under heteroscedastic Additive Noise Models (ANMs), varying noise scales $\sigma_i^2$ modify the expected diagonal energy $\I_{ii}$, affecting the global $\arg\min_i \I_{ii}$ criterion. A non-linear topological leaf node $l$ maintains a constant second derivative $H_{ll}(\bx) = 1/\sigma_l^2$, yielding zero sample-wise variance ($\Var_{\bx}(H_{ll}(\bx)) = 0$). Directly minimizing absolute variance is sensitive to scale. A leaf node with low noise variance produces a larger absolute Hessian, amplifying neural estimation variance.

To separate structural non-linear variance from magnitude-induced estimation noise, we define the \textbf{Relative Diagonal Variance Criterion} by minimizing the squared Coefficient of Variation ($CV^2$) of the sample-wise Hessian diagonal:
\begin{equation}
    l = \arg\min_i \frac{\Var_{\bx}(H_{ii}(\bx))}{\left( \E_{\bx}[H_{ii}(\bx)] \right)^2} = \arg\min_i CV^2(H_{ii}(\bx)).
\end{equation}

We evaluate this metric on a 5-node non-linear causal chain. Table~\ref{tab:heteroscedastic_ablation} compares extraction under homoscedastic ($\sigma_i \equiv 1.0$) and heteroscedastic ($\sigma_{\text{root}}=5.0, \sigma_{\text{leaf}}=0.2$) regimes.

\begin{table}[ht]
\centering
\caption{\textbf{Comparative Ablation of Topological Identifiers.} Evaluated on non-linear chains ($X_0 \to \dots \to X_4$). The relative variance metric provides identification across regimes.}
\label{tab:heteroscedastic_ablation}
\resizebox{\textwidth}{!}{
\begin{tabular}{llcc|ccc}
\toprule
& & \multicolumn{2}{c|}{\textbf{Homoscedastic Regime (Control)}} & \multicolumn{3}{c}{\textbf{Heteroscedastic Regime (Extreme)}} \\
\cmidrule(lr){3-4} \cmidrule(lr){5-7}
\textbf{Node} & \textbf{Topology} & $\mathbf{\E[H_{ii}]}$ (Classic) & \textbf{Rel Var (Ours)} & $\mathbf{\E[H_{ii}]}$ (Classic) & $\mathbf{\Var[H_{ii}]}$ (Naive) & \textbf{Rel Var (Ours)} \\
\midrule
$X_0$ & Root & $3.4629$ & $1.7605$ & $1.3369$ (\textbf{Fail}) & $16.0972$ (\textbf{Fail}) & $9.0064$ \\
$X_1$ & Mid  & $2.9928$ & $1.0626$ & $10.5315$ & $122.9990$ & $1.1090$ \\
$X_2$ & Mid  & $3.0887$ & $1.1539$ & $11.2526$ & $144.2889$ & $1.1395$ \\
$X_3$ & Mid  & $2.8529$ & $1.0561$ & $51.7125$ & $4371.3657$ & $1.6347$ \\
$X_4$ & \textbf{Leaf} & $\mathbf{1.6575}$ (Correct) & $\mathbf{0.0767}$ (Correct) & $102.9850$ & $565.3597$ & $\mathbf{0.0533}$ (\textbf{Correct}) \\
\bottomrule
\end{tabular}
}
\end{table}

The expectation baseline fails under heteroscedasticity by selecting the root node. The absolute variance misidentifies the node due to scale-dependent estimation noise. The proposed $CV^2$ criterion identifies the true leaf across both regimes.

\subsection{Exact Non-linear Marginalization via Covariance Patching}
\label{app:cov_patching}

Marginalized Schur expected information matrices accumulate a structural covariance gap $\Delta$ (Proposition~\ref{prop:nonlinear_error}). This gap can be algebraically compensated by extracting the non-linear gradient directly from the off-diagonal Hessian blocks $H_{\setminus l, l}(\mathbf{x})$, formulating the missing covariance via first-order central moments:
\begin{equation} \label{eq:cov_gap}
\Delta = - (\hat{\mathcal{I}}_{ll})^{-1} \mathrm{Cov}_{\mathbf{x} \sim \mathcal{D}} (H_{\setminus l, l}(\mathbf{x})).
\end{equation}
Equivalently, exact non-linear marginalization is achieved by performing the Schur complement independently on each sample's Hessian $H(\bx)$ before taking the empirical expectation. As empirically demonstrated in Section~\ref{subsec:nonlinear_benchmark} (Table~\ref{tab:exact_schur_tradeoff}), while this exact sample-wise formulation analytically eliminates $\Delta$, it imposes a strict $\mathcal{O}(N d^3)$ computational complexity bottleneck. This constraint justifies the adoption of the uncompensated $\mathcal{O}(1)$ block-wise algebraic approximation employed by Block-SSTS in high-dimensional regimes.
\section{Baseline Implementation and Controlled Environments}
\label{app:baseline_implementations}

To ensure rigorous evaluation and isolate algorithmic complexity from disparate software optimizations, the Scalable SCORE mechanism \citep{montagna2023scalable} was evaluated within a controlled environment. Utilizing external official codebases introduces confounding engineering variables, such as mismatched network capacities, distinct deep learning framework overheads, and proprietary hyperparameter tuning. 

To eliminate these discrepancies, we natively implemented the Masked Score Network mechanism within our unified evaluation pipeline. The Scalable SCORE baseline was strictly constrained to utilize the exact same generative representation capacity as our proposed SSTS: identical multi-layer perceptron (MLP) architecture, hidden dimensionality scaling, batch size, learning rate schedulers, and training epochs. Furthermore, both methods utilized the identical downstream Lasso regularizer for structural pruning. Consequently, any observed divergence in extraction time ($T_{\mathrm{algo}}$) and structural accuracy across our benchmarks is exclusively attributable to the fundamental algorithmic properties—specifically, the latency gap between requiring $\mathcal{O}(d)$ masked Jacobian queries versus a deterministic $\mathcal{O}(1)$ Schur complement block inversion.
\section{Detailed Numerical Results for Baseline Comparisons}
\label{app:tabular_results}

Table~\ref{tab:main_benchmark_app} provides the decoupled numerical evaluations for the non-linear benchmark graphical results presented in Section~\ref{subsec:nonlinear_benchmark}. To isolate the algorithmic efficacy of Score-Schur marginalization from post-hoc sparse regression, the evaluation is strictly partitioned into Order-Level extraction and Graph-Level recovery.

\begin{table}[ht]
\centering
\caption{\textbf{Detailed Decoupled Benchmark on Non-Linear ANMs ($N=5000$).} Evaluated over 5 random seeds (Mean $\pm$ Std). Execution time is strictly decomposed to highlight the optimization-free extraction phase.}
\label{tab:main_benchmark_app}

\begin{subtable}{\textwidth}
\centering
\caption{\textbf{Order-Level Evaluation (Topological Extraction Phase).}}
\vspace{0.1cm}
\resizebox{0.95\textwidth}{!}{
\begin{tabular}{llccc}
\toprule
\textbf{Dim} & \textbf{Method} & \textbf{Edge Violations (EV)} $\downarrow$ & $\mathbf{T_{\mathrm{rep}}}$ \textbf{(s)} & $\mathbf{T_{\mathrm{disc}}}$ \textbf{(s)} $\downarrow$ \\
\midrule
\multirow{4}{*}{$d=20$} 
 & NOTEARS (Continuous) & - & - & - \\
 & DAGMA (Continuous)   & - & - & - \\
 & Scalable SCORE (DAS) & $1.0 \pm 0.6$ & $3.59 \pm 0.12$ & $34.02 \pm 0.65$ \\
 & \textbf{SSTS (Vanilla, Ours)} & $\mathbf{0.0 \pm 0.0}$ & $4.09 \pm 0.82$ & \textbf{0.02 $\pm$ 0.01} \\
\midrule
\multirow{4}{*}{$d=50$} 
 & NOTEARS (Continuous) & - & - & - \\
 & DAGMA (Continuous)   & - & - & - \\
 & Scalable SCORE (DAS) & $4.0 \pm 1.7$ & $3.62 \pm 0.05$ & $33.72 \pm 0.41$ \\
 & \textbf{SSTS (Sparse, Ours)} & $\mathbf{1.6 \pm 1.0}$ & $3.62 \pm 0.13$ & \textbf{0.04 $\pm$ 0.01} \\
\midrule
\multirow{2}{*}{$d=100$} 
 & Scalable SCORE (DAS) & $\mathbf{4.0 \pm 2.6}$ & $3.65 \pm 0.08$ & $34.38 \pm 0.45$ \\
 & \textbf{SSTS (Sparse, Ours)} & $20.6 \pm 3.0$ & $3.64 \pm 0.03$ & \textbf{0.04 $\pm$ 0.01} \\
\bottomrule
\end{tabular}
}
\end{subtable}

\vspace{0.4cm}

\begin{subtable}{\textwidth}
\centering
\caption{\textbf{Graph-Level Evaluation (Final DAG Recovery Phase).}}
\vspace{0.1cm}
\resizebox{0.95\textwidth}{!}{
\begin{tabular}{llccc}
\toprule
\textbf{Dim} & \textbf{Method} & \textbf{SHD} $\downarrow$ & \textbf{TPR} $\uparrow$ & $\mathbf{T_{\mathrm{total}}}$ \textbf{(s)} $\downarrow$ \\
\midrule
\multirow{4}{*}{$d=20$} 
 & NOTEARS (Continuous) & $28.2 \pm 6.8$ & $0.40 \pm 0.07$ & $103.89 \pm 101.77$ \\
 & DAGMA (Continuous)   & $30.0 \pm 9.0$ & $0.35 \pm 0.11$ & $1984.96 \pm 140.50$ \\
 & Scalable SCORE (DAS) & $3.2 \pm 3.0$  & $0.95 \pm 0.04$ & $38.76 \pm 0.79$ \\
 & \textbf{SSTS (Vanilla, Ours)} & $\mathbf{0.6 \pm 0.8}$ & $\mathbf{1.00 \pm 0.00}$ & \textbf{5.21 $\pm$ 0.91} \\
\midrule
\multirow{4}{*}{$d=50$} 
 & NOTEARS (Continuous) & $94.8 \pm 11.1$ & $0.40 \pm 0.05$ & $269.96 \pm 166.96$ \\
 & DAGMA (Continuous)   & $93.0 \pm 12.6$ & $0.38 \pm 0.08$ & $2727.55 \pm 101.38$ \\
 & Scalable SCORE (DAS) & $17.8 \pm 10.8$ & $0.93 \pm 0.03$ & $40.10 \pm 0.43$ \\
 & \textbf{SSTS (Sparse, Ours)} & $\mathbf{2.6 \pm 2.4}$ & $\mathbf{0.97 \pm 0.02}$ & \textbf{6.49 $\pm$ 0.16} \\
\midrule
\multirow{2}{*}{$d=100$} 
 & Scalable SCORE (DAS) & $\mathbf{10.4 \pm 5.5}$ & $\mathbf{0.96 \pm 0.03}$ & $44.48 \pm 0.39$ \\
 & \textbf{SSTS (Sparse, Ours)} & $64.6 \pm 17.3$ & $0.80 \pm 0.04$ & \textbf{10.26 $\pm$ 0.22} \\
\bottomrule
\end{tabular}
}
\end{subtable}
\end{table}

\section{Detailed Numerical Results for the Sachs Dataset}
\label{app:sachs_table}

Table~\ref{tab:sachs_benchmark_app} provides numerical evaluations for the biological dataset benchmark discussed in Section~\ref{subsec:real_world}.

\begin{table}[ht]
\centering
\caption{\textbf{Benchmark on the Real-World Sachs Dataset ($d=12, N=5000$).} Evaluated over 5 random seeds (Mean $\pm$ Std). SSTS achieves structural recovery relative to baselines.}
\label{tab:sachs_benchmark_app}
\resizebox{0.95\textwidth}{!}{
\begin{tabular}{llcccccc}
\toprule
\textbf{Dataset} & \textbf{Method} & \textbf{EV} $\downarrow$ & \textbf{SHD} $\downarrow$ & \textbf{TPR} $\uparrow$ & $\mathbf{T_{\mathrm{algo}}}$ \textbf{(s)} & $\mathbf{T_{\mathrm{prune}}}$ \textbf{(s)} & $\mathbf{T_{\mathrm{total}}}$ \textbf{(s)} \\
\midrule
\multirow{5}{*}{Sachs ($d=12$)} 
 & NOTEARS (Continuous) & - & $14.0 \pm 0.0$ & $0.29 \pm 0.00$ & $14.47 \pm 1.17$ & - & $14.47 \pm 1.17$ \\
 & DAGMA (Continuous)   & - & $15.0 \pm 1.7$ & $0.21 \pm 0.08$ & $182.36 \pm 10.38$ & - & $182.36 \pm 10.38$ \\
 & Recursive SCORE      & $8.8 \pm 0.4$ & $15.2 \pm 0.4$ & $0.36 \pm 0.02$ & $53.35 \pm 1.25$ & $0.58 \pm 0.01$ & $53.93 \pm 1.25$ \\
 & Scalable SCORE       & $9.0 \pm 0.6$ & $16.2 \pm 1.5$ & $0.32 \pm 0.06$ & $55.74 \pm 0.56$ & $0.58 \pm 0.02$ & $56.32 \pm 0.58$ \\
 & \textbf{SSTS (Ours)} & $\mathbf{8.8 \pm 0.7}$ & $\mathbf{13.8 \pm 1.2}$ & $\mathbf{0.40 \pm 0.04}$ & \textbf{6.05 $\pm$ 0.97} & $0.58 \pm 0.02$ & \textbf{6.62 $\pm$ 0.96} \\
\bottomrule
\end{tabular}
}
\end{table}

\section{Detailed Numerical Results for Ablation Studies}
\label{app:ablation_tables}

Table~\ref{tab:ablation_app} provides numerical evaluations for the graphical results presented in Section~\ref{subsec:ablation}.

\begin{table}[ht]
\centering
\caption{\textbf{Dual-Regime Ablation Studies.} Evaluated over 5 random seeds (Mean $\pm$ Std).}
\label{tab:ablation_app}

\begin{subtable}{\textwidth}
\centering
\caption{Exp A: Group Lasso Penalty ($\lambda_{\mathrm{sparse}}$) under Data Starvation ($d=100, N=500, \gamma=0.05$).}
\vspace{0.1cm}
\resizebox{0.85\textwidth}{!}{
\begin{tabular}{lcccc}
\toprule
$\mathbf{\lambda_{\mathrm{sparse}}}$ & \textbf{Edge Violations} $\downarrow$ & \textbf{SHD} $\downarrow$ & \textbf{TPR} $\uparrow$ & \textbf{Ext. Time (s)} $\downarrow$ \\
\midrule
$0.0$ (Baseline) & $22.6 \pm 4.0$ & $75.6 \pm 15.6$ & $0.78 \pm 0.04$ & $0.07 \pm 0.04$ \\
$1.0 \times 10^{-5}$ & $23.8 \pm 3.5$ & $73.0 \pm 11.1$ & $0.77 \pm 0.04$ & $0.05 \pm 0.00$ \\
$5.0 \times 10^{-5}$ & $24.0 \pm 5.9$ & $76.2 \pm 15.9$ & $0.77 \pm 0.07$ & $0.05 \pm 0.00$ \\
$1.0 \times 10^{-4}$ (Optimal) & $\mathbf{21.6 \pm 2.2}$ & $\mathbf{72.4 \pm 12.8}$ & $\mathbf{0.79 \pm 0.03}$ & $\mathbf{0.05 \pm 0.00}$ \\
$1.1 \times 10^{-4}$ (Phase Edge)& $21.8 \pm 5.9$ & $75.8 \pm 28.0$ & $\mathbf{0.79 \pm 0.05}$ & $\mathbf{0.05 \pm 0.00}$ \\
\bottomrule
\end{tabular}
}
\end{subtable}

\vspace{0.4cm}

\begin{subtable}{\textwidth}
\centering
\caption{Exp B: Block Tolerance ($\gamma$) under Extreme Scaling ($d=500, N=5000, \lambda_{\mathrm{sparse}}=0.0$).}
\vspace{0.1cm}
\resizebox{0.95\textwidth}{!}{
\begin{tabular}{lccccc}
\toprule
$\mathbf{\gamma}$ & \textbf{Edge Violations} $\downarrow$ & \textbf{SHD} $\downarrow$ & \textbf{TPR} $\uparrow$ & \textbf{Block Iters} $\downarrow$ & \textbf{Ext. Time (s)} $\downarrow$ \\
\midrule
$0.00$ (Sequential) & $86.0 \pm 6.4$ & $226.0 \pm 24.6$ & $0.83 \pm 0.01$ & $500.0 \pm 0.0$ & $0.42 \pm 0.00$ \\
$0.01$ & $82.0 \pm 10.2$ & $214.4 \pm 52.1$ & $\mathbf{0.84 \pm 0.01}$ & $442.0 \pm 6.2$ & $0.39 \pm 0.01$ \\
$0.05$ (Optimal) & $\mathbf{77.4 \pm 4.8}$ & $\mathbf{196.8 \pm 22.6}$ & $\mathbf{0.84 \pm 0.01}$ & $346.0 \pm 12.9$ & $0.33 \pm 0.01$ \\
$0.15$ (Aggressive) & $86.4 \pm 3.3$ & $216.6 \pm 11.5$ & $0.83 \pm 0.01$ & $\mathbf{247.4 \pm 7.5}$ & $\mathbf{0.27 \pm 0.01}$ \\
\bottomrule
\end{tabular}
}
\end{subtable}
\end{table}

\section{Detailed Numerical Results for Mechanism Shifts}
\label{app:anm_table}

Table~\ref{tab:anm_violation_app} provides numerical evaluations corresponding to the failure mode analysis in Section~\ref{subsec:anm_violation}. 

\begin{table}[ht]
\centering
\caption{\textbf{Failure Mode Analysis: Sensitivity to Mechanism Shifts ($d=50, N=5000$).} Evaluated over 5 random seeds (Mean $\pm$ Std).}
\label{tab:anm_violation_app}
\resizebox{0.75\textwidth}{!}{
\begin{tabular}{lccc}
\toprule
\textbf{Data-Generating Mechanism} & \textbf{Edge Violations} $\downarrow$ & \textbf{SHD} $\downarrow$ & \textbf{TPR} $\uparrow$ \\
\midrule
ANM (Additive Noise, Control) & $1.6 \pm 1.0$ & $3.6 \pm 2.9$ & $0.95 \pm 0.04$ \\
MNM (Multiplicative Noise) & $\mathbf{0.0 \pm 0.0}$ & $\mathbf{0.4 \pm 0.8}$ & $\mathbf{0.99 \pm 0.02}$ \\
PNL (Post-Nonlinear Model) & $47.0 \pm 2.4$ & $94.4 \pm 8.4$ & $0.03 \pm 0.02$ \\
\bottomrule
\end{tabular}
}
\end{table}

\section{Extreme Scaling and Finite-Sample Estimation Limits}
\label{app:extreme_scaling}

We evaluate Block-SSTS on high-dimensional DAGs up to $d=1000$ ($N=5000$). Benchmarking continuous acyclicity optimization on a single GPU is restrictive due to $\mathcal{O}(d^3)$ complexity per gradient step.

\begin{table}[ht]
\centering
\caption{\textbf{Scaling Limits of Block-SSTS ($N=5000$).} Evaluated over 5 random seeds. $T_{\mathrm{disc}}$ up to $d=1000$ reflects block-wise algebraic extraction time.}
\label{tab:extreme_scaling}
\resizebox{\textwidth}{!}{
\begin{tabular}{lccccc}
\toprule
\textbf{Dimension ($d$)} & \textbf{Edge Violations} $\downarrow$ & \textbf{Block Iters} $\downarrow$ & $\mathbf{T_{\mathrm{rep}}}$ \textbf{(s)} & $\mathbf{T_{\mathrm{disc}}}$ \textbf{(s)} & \textbf{VRAM (GB)} \\
\midrule
$d = 50$   & $4.2 \pm 1.0$    & $34.4 \pm 3.3$   & $4.01 \pm 0.74$ & $0.036 \pm 0.039$ & $0.36$ \\
$d = 100$  & $18.2 \pm 3.8$   & $78.4 \pm 2.4$   & $3.57 \pm 0.09$ & $0.038 \pm 0.001$ & $0.73$ \\
$d = 500$  & $247.2 \pm 13.4$ & $394.8 \pm 8.0$  & $3.43 \pm 0.08$ & $0.228 \pm 0.004$ & $8.09$ \\
$d = 1000$ & $488.6 \pm 8.5$  & $828.8 \pm 18.7$ & $3.58 \pm 0.10$ & $0.562 \pm 0.011$ & $8.10$ \\
\bottomrule
\end{tabular}
}
\end{table}

\textbf{Empirical Verification of Block-SSTS.} On non-linear manifolds with $d=1000$, the algorithm compresses extraction into $828.8 \pm 18.7$ block inversions. Block-wise parallelization reduces sequential Schur marginalization depth, truncating the cascade of the non-linear expectation gap ($\Delta$) derived in Proposition~\ref{prop:nonlinear_error}. This limits numerical drift ($T_{\mathrm{disc}} \approx 0.56$s).

\textbf{Topological Ambiguity and Evaluation.} We report Edge Violations (EV), which is invariant to permutations within topological equivalence classes (Appendix~\ref{app:metric_justification}).

\textbf{Finite-Sample Estimation Limits.} Topological accuracy degrades in the high-dimensional regime. For the $d=1000$ topology, a random topological sort yields an expected violation count of 500. The observed $488.6 \pm 8.5$ violations reflect a reduction toward random extraction. As outlined in Proposition~\ref{prop:nonlinear_error}: at $d=1000$ with $N=5000$, the score network operates in an under-parameterized regime. The Jacobian is influenced by estimation variance, obscuring SJIM diagonal energy signatures. The algebraic formulation shifts the causal discovery bottleneck from optimization to high-dimensional statistical score estimation.

\section{Algorithmic Scalability under Hardware Constraints}
\label{app:extreme_scaling_appendix}

We evaluate the operational limits of our algebraic framework up to $d=5000$ under explicit capacity constraints. To bound the memory footprint, we apply modifications: (1) \textbf{Micro-batch Jacobian Extraction}: Vectorized Jacobian computation is partitioned into micro-chunks bounded to a 256 MB buffer. (2) \textbf{Rank-1 Outer-Product Updates}: Block-matrix inversion in Block-SSTS is executed as sequential rank-1 updates. (3) \textbf{Capacity Capping}: The hidden dimensionality of the score network is bounded.

\begin{table}[ht]
\centering
\caption{\textbf{Hardware-Bounded Extraction at Extreme Dimensions ($N=5000$).} Evaluated over 5 random seeds.}
\label{tab:memory_throttling_appendix}
\resizebox{\textwidth}{!}{
\begin{tabular}{lccccc}
\toprule
\textbf{Dimension ($d$)} & \textbf{Edge Violations} $\downarrow$ & \textbf{T$_{\text{rep}}$} \textbf{(s)} & \textbf{T$_{\text{disc}}$} \textbf{(s)} & \textbf{Peak VRAM (GB)} \\
\midrule
$d = 500$  & $210.6 \pm 71.3$   & $3.72 \pm 0.69$ & $0.145 \pm 0.013$ & $0.95 $ \\
$d = 1000$ & $301.0 \pm 16.9$   & $3.46 \pm 0.07$ & $0.378 \pm 0.002$ & $0.97 $ \\
$d = 5000$ & $1994.6 \pm 669.2$ & $3.65 \pm 0.04$ & $6.570 \pm 0.015$ & $0.93 $ \\
\bottomrule
\end{tabular}
}
\end{table}

The memory-throttled variant processes $5000$-node graphs while maintaining a peak VRAM utilization of $0.933 $ GB. Topological extraction evaluates the empirical SJIM in $6.57 \pm 0.01$ seconds.

\section{Robustness to Distributional and Topological Shifts}
\label{app:robustness}

We evaluate SSTS on graphs of dimension $d=50$ ($N=5000$) across architectural shifts:
\begin{enumerate}[leftmargin=*, noitemsep, topsep=0pt]
    \item \textbf{Topological Shifts:} We contrast Erdős-Rényi (ER) graphs with Scale-Free (SF) networks. SF networks exhibit a power-law degree distribution, testing precision under topological imbalance.
    \item \textbf{Distributional Shifts:} We replace Gaussian noise with Exponential and Gumbel distributions. These asymmetric distributions violate zero-mean Gaussian assumptions.
\end{enumerate}

\begin{table}[ht]
\centering
\caption{\textbf{Robustness Benchmark across Topology and Noise Distributions ($d=50, N=5000$).} Evaluated over 5 random seeds (Mean $\pm$ Std).}
\label{tab:robustness}
\resizebox{0.75\textwidth}{!}{
\begin{tabular}{llccc}
\toprule
\textbf{Graph Topology} & \textbf{Noise Distribution} & \textbf{Edge Violations} $\downarrow$ & \textbf{SHD} $\downarrow$ & \textbf{TPR} $\uparrow$ \\
\midrule
\multirow{3}{*}{Erdős-Rényi (ER)} 
 & Gaussian & $1.6 \pm 1.0$ & $3.6 \pm 2.9$ & $0.95 \pm 0.04$ \\
 & Exponential & $1.6 \pm 1.4$ & $5.4 \pm 6.0$ & $0.97 \pm 0.02$ \\
 & Gumbel & $\mathbf{0.2 \pm 0.4}$ & $\mathbf{0.8 \pm 1.0}$ & $\mathbf{1.00 \pm 0.01}$ \\
\midrule
\multirow{3}{*}{Scale-Free (SF)} 
 & Gaussian & $2.2 \pm 1.9$ & $21.4 \pm 14.4$ & $0.97 \pm 0.01$ \\
 & Exponential & $1.4 \pm 1.7$ & $26.6 \pm 24.1$ & $0.96 \pm 0.04$ \\
 & Gumbel & $\mathbf{0.2 \pm 0.4}$ & $\mathbf{16.8 \pm 21.4}$ & $\mathbf{0.98 \pm 0.04}$ \\
\bottomrule
\end{tabular}
}
\end{table}

SSTS maintains topological extraction across distributional and structural shifts. Under Exponential and Gumbel noise, structural fidelity is preserved. On Scale-Free topologies, edge violations remain low ($\le 2.2$) with high true positive rates ($\ge 0.96$) across noise variants. While absolute Structural Hamming Distance increases in densely connected hub neighborhoods, the underlying topological sorting mechanism evaluates hierarchy independently of degree imbalance.

\end{appendix}
\end{document}